\documentclass[conference]{IEEEtran}
\IEEEoverridecommandlockouts
\usepackage{cite}
\usepackage{amsmath,amssymb,amsfonts}
\usepackage{algorithmic}
\usepackage{graphicx}
\usepackage{textcomp}
\usepackage{xcolor}
\usepackage{booktabs}
\usepackage{multirow}
\usepackage{amsthm}   
\usepackage{stfloats}
\def\BibTeX{{\rm B\kern-.05em{\sc i\kern-.025em b}\kern-.08em
    T\kern-.1667em\lower.7ex\hbox{E}\kern-.125emX}}

\newtheoremstyle{mythmstyle}
  {3pt}          
  {3pt}          
  {\normalfont}  
  {0pt}          
  {\bfseries}    
  {.}            
  {.5em}         
  {}             

\theoremstyle{mythmstyle}
\newtheorem{theorem}{Theorem}

\theoremstyle{mythmstyle}
\newtheorem{assumption}{Assumption}

\makeatletter
\renewenvironment{proof}[1][\proofname]{\par
  \pushQED{\qed}%
  \normalfont \topsep6\p@\@plus6\p@\relax
  \trivlist
  \item[\hskip\labelsep\bfseries #1\@addpunct{.}]\ignorespaces 
}{\popQED\endtrivlist\@endpefalse}
\makeatother

\begin{document}

\title{Texture-preserving implicit neural representation for Cone beam CT truncated reconstruction\\
\thanks{Identify applicable funding agency here. If none, delete this.}
}

\author{\IEEEauthorblockN{1\textsuperscript{st} Given Name Surname}
\IEEEauthorblockA{\textit{dept. name of organization (of Aff.)} \\
\textit{name of organization (of Aff.)}\\
City, Country \\
email address or ORCID}
\and
\IEEEauthorblockN{2\textsuperscript{nd} Given Name Surname}
\IEEEauthorblockA{\textit{dept. name of organization (of Aff.)} \\
\textit{name of organization (of Aff.)}\\
City, Country \\
email address or ORCID}
\and
\IEEEauthorblockN{3\textsuperscript{rd} Given Name Surname}
\IEEEauthorblockA{\textit{dept. name of organization (of Aff.)} \\
\textit{name of organization (of Aff.)}\\
City, Country \\
email address or ORCID}
\and
\IEEEauthorblockN{4\textsuperscript{th} Given Name Surname}
\IEEEauthorblockA{\textit{dept. name of organization (of Aff.)} \\
\textit{name of organization (of Aff.)}\\
City, Country \\
email address or ORCID}
\and
\IEEEauthorblockN{5\textsuperscript{th} Given Name Surname}
\IEEEauthorblockA{\textit{dept. name of organization (of Aff.)} \\
\textit{name of organization (of Aff.)}\\
City, Country \\
email address or ORCID}
\and
\IEEEauthorblockN{6\textsuperscript{th} Given Name Surname}
\IEEEauthorblockA{\textit{dept. name of organization (of Aff.)} \\
\textit{name of organization (of Aff.)}\\
City, Country \\
email address or ORCID}
}

\title{Texture-preserving implicit neural representation for Cone beam CT truncated reconstruction\\
\thanks{This work was supported by the National Key Research and Development Program of China (No. 2022YFF0706400), the National Natural Science Foundation of China (No. 62171067), and the Fundamental Research Funds for the Central Universities (No. 2024CDJYXTD-009). 

Genyuan Zhang and Junyao Wang contributed equally to this work. 

Corresponding authors: Fenglin Liu}
}

\author{
    \IEEEauthorblockN{
        Genyuan Zhang\textsuperscript{1}, 
        Junyao Wang\textsuperscript{1}, 
        Haoran Lan\textsuperscript{1}, 
        Chuandong Tan\textsuperscript{1}, \\ 
        Songtao Zhu\textsuperscript{1}, 
        Fenglin Liu\textsuperscript{1},  
    }
    \vspace{0.15cm} 
    
    \IEEEauthorblockA{
        \textsuperscript{1}\textit{Key Lab of Optoelectronic Technology and Systems, and Engineering Research Center of} \\
        \textit{Industrial Computed Tomography Nondestructive Testing, Ministry of Education} \\
        \textit{Chongqing University}, Chongqing, China \\
        Genyuan Zhang: zhanggy@stu.cqu.edu.cn; Junyao Wang: wjy19972405100@163.com; \\
        Haoran Lan: 202508021074T@stu.cqu.edu.cn; Chuandong Tan: tancd@stu.cqu.edu.cn; \\
        Songtao Zhu: 202408131056T@stu.cqu.edu.cn; Fenglin Liu: liufl@cqu.edu.cn
    }
    \vspace{0.15cm}
    
}

\maketitle

\begin{abstract}
Cone-beam computed tomography (CBCT) frequently suffers from data truncation, which introduces severe artifacts and limits the effective field of view (FOV). Existing deep learning methods for truncated cone-beam computed tomography (CBCT) reconstruction suffer from serious limitations, including a strict reliance on supervised ground truth and a failure to account for continuous 3D spatial truncation variations. To address these challenges, we introduce a self-supervised 3D reconstruction framework based on neural scene representations. By directly mapping spatial coordinates to radiodensity under projection supervision, our approach inherently bypasses traditional filtering and backprojection operations, thereby fundamentally eliminating truncation-induced ring artifacts while enabling robust continuous 3D data extrapolation. However, coordinate networks are susceptible to an inherent spectral bias, which leads to a severe loss of clinically vital high-frequency textures. To resolve this bottleneck, we further incorporate a physics-based iterative refinement module into the neural scene representation architecture. Leveraging the artifact-free, extrapolated volume from the coordinate network as an optimal initialization, this module progressively re-extracts and injects high-frequency structural information from the original projections back into the volume. Extensive experiments on both simulated and real-world datasets demonstrate that our method successfully unifies the exceptional artifact suppression and extrapolation capabilities of neural networks with the high-fidelity detail preservation of iterative algorithms.
\end{abstract}


\section{Introduction}
Cone-beam Computed Tomography (CBCT) enables three-dimensional anatomical imaging and is one of the key clinical imaging tools. However, its diagnostic efficacy is frequently compromised by data truncation. This phenomenon typically arises in two scenarios: when the imaging field is intentionally restricted to a region of interest (ROI) to minimize radiation dose \cite{yu2006region,sen20123d}, or when a patient's anatomy exceeds the detector's limited field of view (FOV) \cite{ohnesorge2000efficient, modica2011obese}.

Standard analytical algorithms, such as filtered backprojection (FDK), assume that projection data smoothly decays to zero at the detector boundaries. Violating this assumption generates severe, bright truncation ring artifacts \cite{ohnesorge2000efficient}. To address this, numerous data extrapolation methods have been proposed to pad truncated projections \cite{ohnesorge2000efficient,hsieh2004novel,sourbelle2005reconstruction}, alongside techniques that modify the filtering stage to suppress artifacts \cite{xia2013towards,yu2006region,pack2005cone}. Subsequently, compressed sensing approaches, notably total variation (TV) minimization, have been leveraged to improve ROI reconstruction \cite{yu2009compressed,yang2010high,xu2011statistical}. Nevertheless, these traditional methods either rely on hand-crafted priors or only enhance image quality strictly within the FOV, ultimately falling short of a comprehensive solution to the truncation problem.

Deep learning (DL) methods have emerged as a promising alternative, broadly categorized into two paradigms \cite{huang2021data}. The first encompasses image-domain post-processing algorithms, which learn to map artifact-corrupted FBP or DBP reconstructions to truncation-free images in an end-to-end manner \cite{han2019one,fourniect,liman2024diffusion}. However, these approaches often fail to guarantee data consistency with the original measurements. The second paradigm reconstructs images directly from projection data by unrolling analytical or iterative operators into the neural network architecture. While these methods enforce data consistency and achieve impressive results \cite{li2019learning,huang2020field,li2019promising}, they still face critical bottlenecks.

Specifically, existing DL-based methods face the following challenges: 1) These methods rely on untruncation ground truth, which limits their application in scenarios where ground truth is lacking. Furthermore, they present challenges in generalization across different acquisition devices \cite{fawzi2017robustness}. 2) As analyzed in our methodology section regarding the truncation FOV of cone-beam CT, the truncation of a three-dimensional volume changes with its distance from the central plane. Therefore, considering the CBCT truncation problem solely from a 2D perspective has limitations.

Recently, the emergence of neural scene representations—encompassing Neural Radiance Fields (NeRF) or 3D Gaussian Splatting—has provided inspiration for solving the CT truncation problem. We found that these methods, by directly mapping 3D spatial coordinates to their corresponding attenuation coefficients under projection supervision \cite{zang2021intratomo,zha2022naf,cai2024structure}, essentially bypass traditional filtering and backprojection operations, thereby fundamentally eradicating truncation-induced ring artifacts. Furthermore, because the measured X-ray intensity accumulates attenuation from all structures along the ray path (including areas outside the FOV), the network can naturally reconstruct anatomical structures beyond the FOV without additional supervision signals.

\begin{figure}[htbp] 
  \centerline{\includegraphics[width=0.9\linewidth]{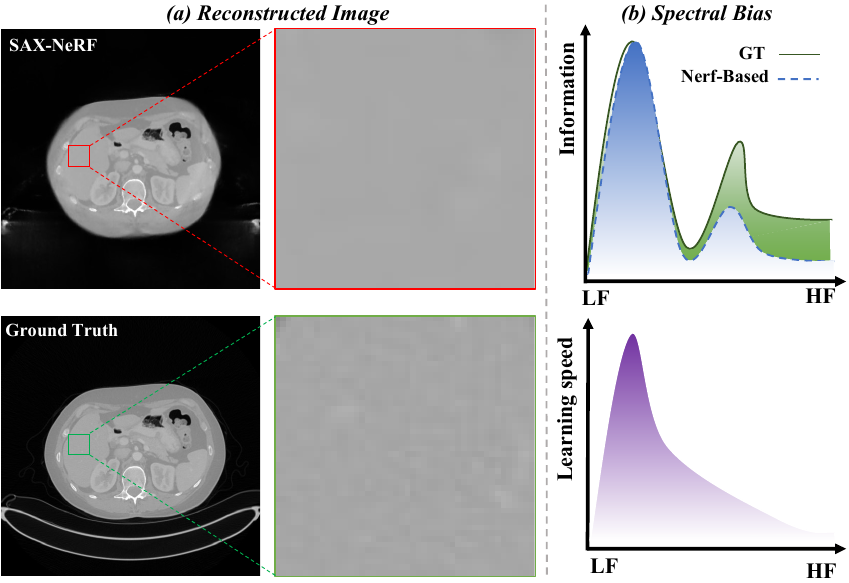}}
  \caption{(a) Representative reconstruction results based on the NeRF method (SAX-NeRF) and ground truth (GT); (b) Schematic diagram of spectral bias: The figure above shows that the difference between the NeRF-based results and the GT within the field of view is mainly concentrated in the high frequency, while the figure below shows that the coordinate network learns high frequencies much slower than it learns low frequencies.}
  \label{bias}
\end{figure}

While neural scene representations excel at artifact suppression and spatial extrapolation, they suffer from a severe loss of clinically vital texture details, as illustrated in Fig. \ref{bias}(a). This degradation stems from the inherent spectral bias of coordinate-based networks \cite{yuce2022structured}: during training, they rapidly converge on low-frequency structural components while struggling to capture high-frequency textural variations (Fig. \ref{bias}(b)). To alleviate this in sparse-view CT reconstruction, various strategies—such as positional encoding, frequency-dependent activation functions, and normalization techniques—have been proposed \cite{sitzmann2020implicit,ramasinghe2022beyond,tancik2020fourier,cai2024batch}. These methods essentially modulate the eigenvalues of the Neural Tangent Kernel (NTK) to accelerate high-frequency learning. However, unlike sparse-view scenarios, truncated projections contain complete and dense high-frequency sampling within the effective FOV. Purely modifying the network's internal representation fails to fully exploit this explicitly measured high-frequency data. On the other hand, since coordinate networks have an inherent spectral bias problem, can we use a method without spectral bias to supplement the missing high-frequency information in the coordinate network reconstruction results?

Based on the above in-depth analysis, we propose a completely different approach. Specifically, we introduce an iteratively refined neural scene representation framework. Specifically, we utilize the artifact-free, extrapolated volume synthesized by the coordinate network as an optimal initialization for a traditional iterative solver (e.g., the Simultaneous Iterative Reconstruction Technique (SIRT) \cite{gilbert1972iterative}). Through subsequent iterative updates driven by the original truncated projections, the high-frequency details are progressively recovered. This seamless coupling effectively restores high-fidelity textures while strictly preserving the robust out-of-FOV extrapolation and artifact suppression achieved by the neural representation.

In summary, we make the following contributions,

\begin{itemize}
    \item We propose a method based on neural scene representations that can fundamentally avoid truncation artifacts while possessing extremely strong and reliable data extrapolation capabilities.
    \item We propose a novel architecture that greatly alleviates the problem of missing high-frequency information in coordinate networks during truncation reconstruction by adding a simple iterative refinement step.
    \item Our simulation experiments, using different phantoms and different FOV sizes, comprehensively demonstrated the feasibility of the method, and further proved its practicality through real-world experiments.
\end{itemize}

\section{Related Work}

\subsection{Neural Scene Representations}
Neural scene representation (NeSR) methods can be categorized into three types. First, implicit neural representation methods utilize MLPs to directly map spatial coordinates into medical signals for 2D medical inverse problems \cite{zang2021intratomo}. Second, neural radiation fields, by introducing physical models and differentiable volume rendering, can directly reconstruct 3D from sparse 2D projections \cite{corona2022mednerf,ruckert2022neat,zha2022naf,cai2024structure}. However, the massive computational cost of ray sampling results in extremely long training and reconstruction times. Third, Gaussian splatting techniques abandon the cumbersome implicit MLP evaluation, employing learnable 3D Gaussian ellipsoids to explicitly represent anatomical structures, combined with fast rasterization techniques, achieving extremely fast reconstruction speeds \cite{kerbl20233d, zha2024r,yuluo2025gr}. However, all of the above methods inevitably lose texture details in medical CT images when applied to truncated reconstruction.
\subsection{Spectral Bias}
Coordinate networks inherently suffer from spectral bias \cite{tancik2020fourier, rahaman2019spectral}, exhibiting a strong tendency to prioritize low-frequency components while failing to capture high-frequency details. From the perspective of Neural Tangent Kernel (NTK) theory, let $K = Q\Lambda Q^T$ denote the eigen-decomposition of the positive semi-definite NTK matrix, where $\Lambda$ contains eigenvalues $\lambda_i \ge 0$. The training error dynamics at iteration $t$ can be approximated as:
\begin{equation}
Q^T(\hat{y}^{(t)}_{train} - y) \approx -e^{-\eta \Lambda t} Q^T y,
\end{equation}
This formulation indicates that the absolute error along the $i$-th eigen-basis decays exponentially at a rate of $\eta \lambda_i$. For conventional MLPs using activations like ReLU, the eigenvalues $\lambda_i$ corresponding to high-frequency components decay precipitously, leading to stalled convergence. To circumvent this limitation, existing literature predominantly explores two paradigms, which fundamentally operate on different terms of this NTK formulation:

\noindent\textbf{Signal Re-organization}: Alternatively, grid-based methods map complex inputs into localized, low-frequency representations via learnable hash tables or multi-scale grids (e.g., InstantNGP \cite{muller2022instant}, NIK-MRI\cite{huang2023neural}, PIXEL \cite{kang2023pixel}, and DINER \cite{zhu2024disorder}). In the NTK context, this mechanism effectively reconstructs the target signal $y$ locally, shifting the signal's energy onto the principal eigen-bases (columns of $Q$) where $\lambda_i$ are inherently large. While this elegantly bypasses the small-$\lambda_i$ penalty to achieve high-fidelity representation, it fundamentally sacrifices interpolation smoothness and often necessitates auxiliary regularization \cite{zhu2026rhino}.

\noindent\textbf{Modifying the Eigenvalue Distribution $\lambda_i$}: This paradigm conceptualizes coordinate networks as a function expansion process using pre-encoded bases, such as Fourier features \cite{tancik2020fourier,landgraf2022pins}, higher-order polynomials \cite{singh2023polynomial}, or periodic/wavelet activations \cite{liu2024finer, ramasinghe2022beyond, sitzmann2020implicit}. Theoretically, these operations explicitly alter the input or activation space to reshape the NTK matrix $K$. By amplifying the eigenvalues $\lambda_i$ associated with high-frequency bases, they directly accelerate the convergence rate $\eta \lambda_i$. However, these methods remain highly sensitive to heuristic hyperparameter tuning and frequency configurations. Another special approach is to use classical normalization techniques to improve the ill-conditioned distribution of NTK feature kernels \cite{cai2024batch,cai2026towards}.

Unlike input mapping and strategies to improve eigenvalue distribution, we propose an iterative post-processing strategy to mitigate the impact of spectral bias on truncated reconstruction.

\section{Methodology}

\subsection{Problem Analysis}
Let $\mathcal{X}$ and $\mathcal{Y}$ denote the 3D image and 2D projection domains. The CT forward model is governed by the X-ray transform $\mathcal{A}: \mathcal{X} \rightarrow \mathcal{Y}$. Under a truncated field of view (FOV), the measurements are masked by a binary diagonal operator $\mathcal{T}$. The truncated inverse problem aims to recover the volume $\mathbf{x} \in \mathcal{X}$ given:
\begin{equation}
    \mathbf{y}_{\text{trunc}} = \mathcal{T}\mathcal{A}\mathbf{x} + \mathbf{\epsilon},
\end{equation}
where $\mathbf{\epsilon}$ is the measurement noise. The non-trivial null space of the composite operator $\mathcal{A}_\mathcal{T} = \mathcal{T}\mathcal{A}$ makes this problem severely ill-posed, inevitably causing artifacts in conventional analytic reconstructions.


A critical, yet often overlooked, characteristic of the CBCT truncation problem is the spatial variability of the FOV along the longitudinal axis.  As shown in Fig \ref{fov} (b), SOD and DOD represent the distances between the X-ray source and detector and the object, respectively; $2d$ represents the detector width; and $\alpha$ represents the X-ray cone angle. For two fields of view (FOV1 and FOV2) located at distances $l_1$ and $l_2$ from the focal plane, their field of view radii are:
\begin{equation}
\begin{array}{l}
R1 = \frac{1}{{\tan \alpha }}(\frac{{{\rm{SOD}}}}{{{\rm{SOD}} + {\rm{DOD}}}} \cdot d - {l_1})\\
R2 = \frac{1}{{\tan \alpha }}(\frac{{SOD}}{{{\rm{SOD}} + {\rm{DOD}}}} \cdot d - {l_2})
\end{array}.
\end{equation}
This geometric relationship rigorously demonstrates that the effective FOV radius is not constant; rather, it decreases linearly as the distance from the central focal plane increases. Consequently, treating CBCT truncation merely as a uniform 2D slice-by-slice problem is inherently limited.

\begin{figure}[htbp] 
  \centerline{\includegraphics[width=\linewidth]{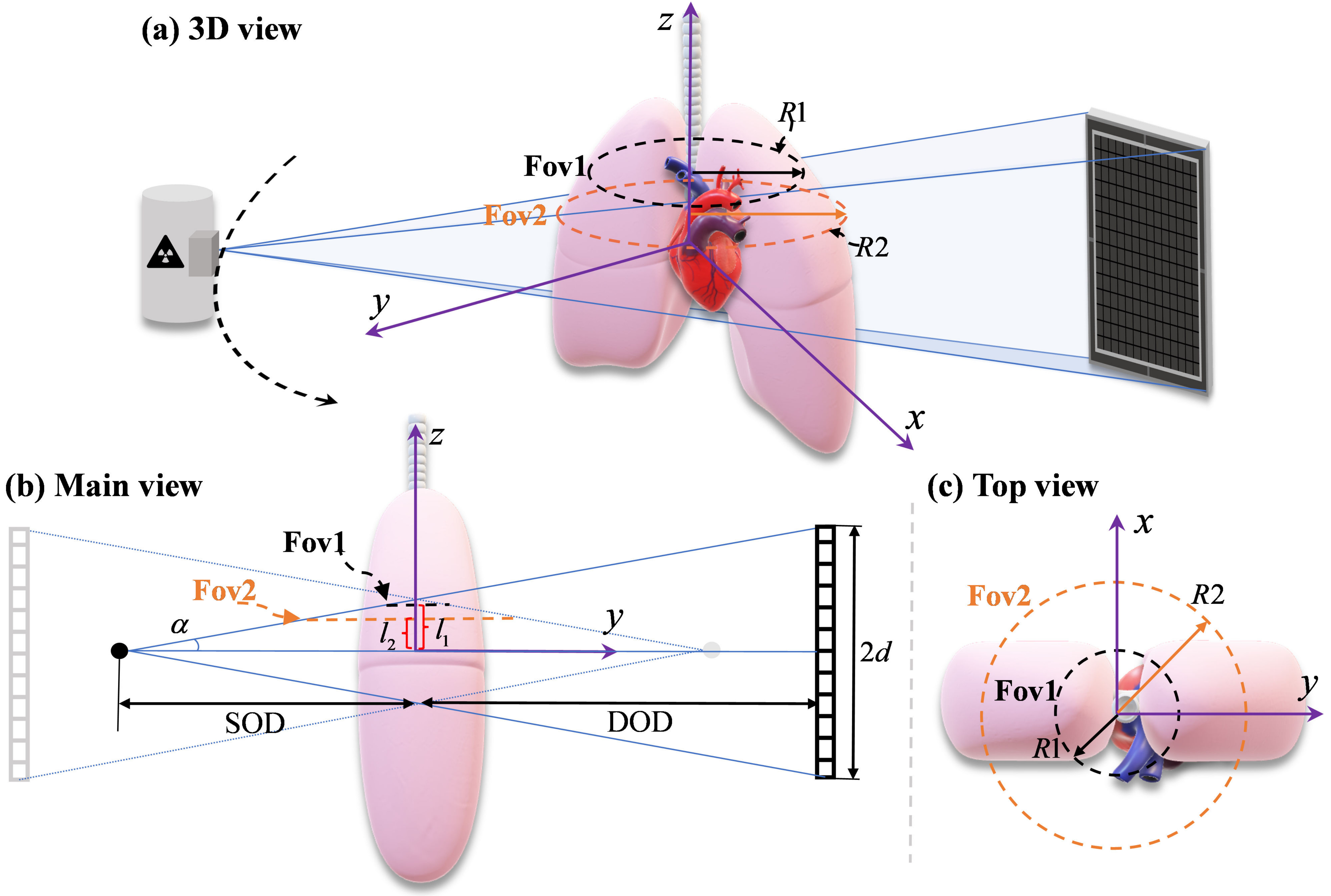}}
  \caption{Field of view analysis: (a) 3D view; (b) Main view; (c) Top view }
  \label{fov}
\end{figure}

\begin{figure*}[htbp] 
  \centerline{\includegraphics[width=\linewidth]{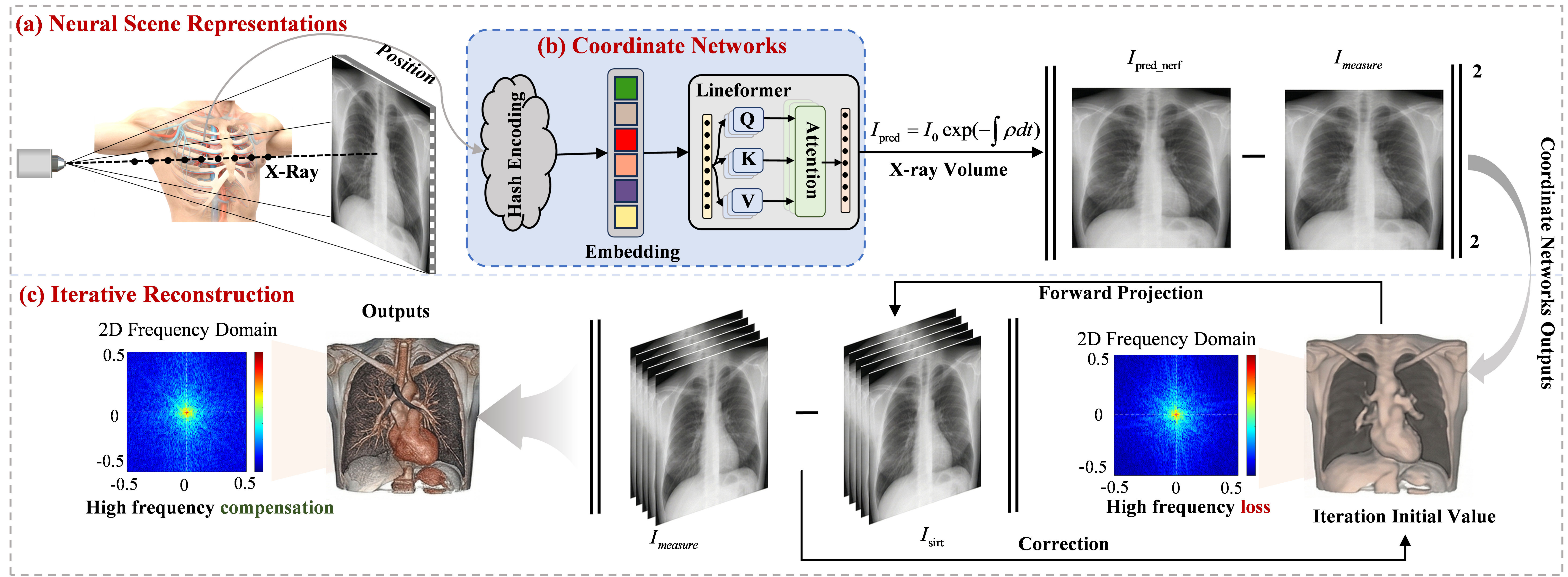}}
  \caption{The overall pipeline of our method. In (a) neural scene representations, a conical beam X-ray scan obtains the measured truncated projection, and then the coordinate information is input into (b) coordinate network. The results of the coordinate network are used as the initial values for (c) iterative reconstruction. The lost high-frequency details are gradually recovered through iteration.}
  \label{overall}
\end{figure*}

\subsection{Overall Framework}
Fig. \ref{overall} illustrates the proposed workflow. As depicted in Fig. \ref{overall}(a), the framework begins with a standard cone-beam X-ray scanning geometry, where the scanner emits cone-shaped beams to capture truncated projections. These projections serve as the supervisory signal for a coordinate-based neural network (Fig. \ref{overall}(b)), which can be instantiated by architectures such as a Multilayer Perceptron (MLP) or Transformer. The network maps 3D spatial coordinates directly to their corresponding attenuation coefficients. We formulate this neural radiodensity field as:
\begin{equation}
    F_\theta : (x, y, z) \rightarrow \rho,
\end{equation}
where $F_\theta$ denotes the continuous mapping function parameterized by trainable weights $\theta$, and $\rho \in \mathbb{R}^+$ represents the predicted radiodensity (attenuation value). 

Governed by the Beer-Lambert law, X-ray intensity attenuates exponentially as it traverses an object. For a given X-ray ray $r(t) = \mathbf{o} + t\mathbf{d} \in \mathbb{R}^3$ bounded by near and far integration limits $t_n$ and $t_f$, the measured projection intensity $\mathbf{I}_{\text{measure}}(r)$ with an initial source intensity $\mathbf{I}_0$ is mathematically modeled as:
\begin{equation}
    \mathbf{I}_{\text{measure}}(r) = \mathbf{I}_0 \cdot \exp\left(-\int_{t_n}^{t_f} \rho(r(t)) dt\right).
    \label{eq1}
\end{equation}
Although the captured projections are spatially truncated due to the limited detector size, the measured intensity values remain physically accurate within the detector bounds, thereby providing valid supervisory signals. By discretizing Eq. \ref{eq1} along the ray path, we derive the predicted projection intensity $\mathbf{I}_{\text{pred}}(r)$:
\begin{equation}
    \mathbf{I}_{\text{pred}}(r) = \mathbf{I}_0 \cdot \exp\left(-\sum_{i=1}^N \rho_i \delta_i\right),
\end{equation}
where $\rho_i$ is the network-predicted radiodensity at the $i$-th sampled location along the ray, and $\delta_i = \|\mathbf{p}_{i+1} - \mathbf{p}_i\|_2$ denotes the Euclidean distance between adjacent sampling points. 

To optimize the network, we define a photometric loss function $\mathcal{L}$ that minimizes the squared error between the predicted and measured intensities across a sampled ray batch $\mathcal{R}$:
\begin{equation}
    \mathcal{L} = \sum_{r\in\mathcal{R}} \|\mathbf{I}_{\text{pred}}(r) - \mathbf{I}_{\text{measure}}(r)\|_2^2.
\end{equation}

Crucially, the sampled 3D coordinates are not strictly bounded by the nominal FOV. Because the measured X-ray intensities accumulate attenuation from all structures along the ray path (including out-of-FOV regions), the network inherently reconstructs the broader anatomical context.  Furthermore, since the neural scene representations method does not involve a back-projection process, it naturally avoids the problem of truncated bright rings.

However, coordinate networks are susceptible to spectral bias, leading to the loss of high-frequency textures. To address this, the synthesized volume from the trained coordinate network is subsequently fed into a physics-based iterative architecture (Fig. \ref{overall}(c)) as a high-quality initialization. Specifically, we employ the SIRT to further refine the volume. By iteratively minimizing the projection residual between the forward-projected initialization and the measured projections, this module progressively re-injects missing high-frequency details into the volume while strictly preserving the artifact-free structural integrity. Detailed mathematical proofs regarding this spectral refinement will be provided in the subsequent sections.

\subsection{Spectral complementarity}

The iterative refinement module effectively recovers the high-frequency information missing from the initial coordinate-based neural network reconstruction. To mathematically formalize this complementary relationship, we analyze the truncated forward operator via its Singular Value Decomposition (SVD).

Let $\mathcal{A}_\mathcal{T} \in \mathbb{R}^{M \times N}$ denote the truncated forward projection matrix. Its compact SVD is given by $\mathcal{A}_\mathcal{T} = \mathbf{U} \mathbf{\Sigma} \mathbf{V}^\top$, where $\mathbf{U}$ and $\mathbf{V}$ possess orthonormal columns, and $\mathbf{\Sigma} = \operatorname{diag}(\sigma_1, \sigma_2, \dots, \sigma_r)$ with strictly positive singular values ordered decreasingly ($\sigma_1 \ge \sigma_2 \ge \dots \ge \sigma_r > 0$). Given a predefined truncation rank $k \in \mathbb{N}^+$, we partition the image and measurement spaces into low-order and high-order subspaces, denoted by subscripts `low' and `high' respectively (e.g., $\mathbf{V} = [\mathbf{V}_{\text{low}}, \mathbf{V}_{\text{high}}]$). The rank $k$ is chosen such that the first $k$ singular values capture the vast majority of the projection energy.

\begin{assumption}[Low-order approximation property of NeSR]
Coordinate-based neural networks exhibit a well-documented spectral bias \cite{yuce2022structured}. In the projection domain, the NeSR reconstruction $\mathbf{x}_{\text{NeSR}}$ is highly consistent with the truncated measurement $\mathbf{y}_{\text{trunc}}$ within the low-order measurement subspace $\mathbf{U}_{\text{low}}$. Formally, there exists a sufficiently small constant $\varepsilon > 0$ such that the initial residual $\mathbf{r}_0 = \mathbf{y}_{\text{trunc}} - \mathcal{A}_\mathcal{T} \mathbf{x}_{\text{NeSR}}$ satisfies:
\begin{equation}
    \|\mathbf{U}_{\text{low}}^\top \mathbf{r}_0\| \le \varepsilon \|\mathbf{r}_0\|.
\end{equation}
\end{assumption}

Based on this structural property, we demonstrate that the initial gradient of the subsequent iterative solver is heavily concentrated in the high-frequency subspace, naturally complementing the NeSR's deficiency.

\begin{theorem}[Spectral Concentration of the Initial Gradient]
\label{thm:spectral_concentration}
Under Assumption 1, if the condition number of the truncated operator $\kappa = \sigma_1 / \sigma_r$ satisfies $\kappa < \sqrt{1-\varepsilon^2}/\varepsilon$, the initial gradient $\mathbf{g}_0$ of the iterative solver is strictly dominated by its high-frequency components (i.e., $\|\mathbf{V}_{\text{high}}^\top \mathbf{g}_0\| > \|\mathbf{V}_{\text{low}}^\top \mathbf{g}_0\|$).
\end{theorem}

The detailed mathematical proof of Theorem \ref{thm:spectral_concentration} is provided in Appendix A. This theorem guarantees that the magnitude of the initial gradient is decisively driven by high-frequency errors, allowing the iterative solver to specifically target and recover the missing localized details without severely altering the well-fitted global structures.

\begin{figure*}[!b] 
  \centerline{\includegraphics[width=\linewidth]{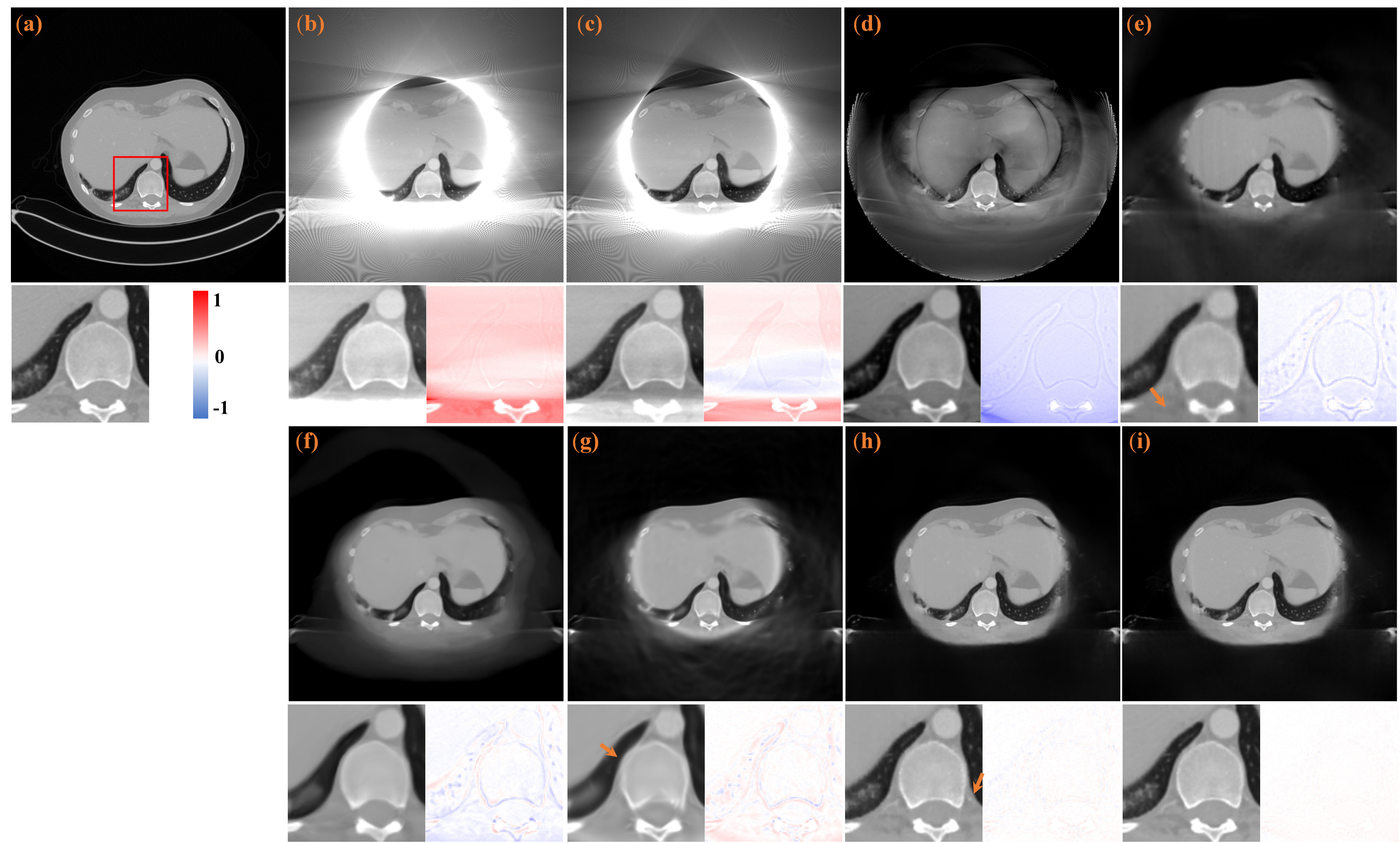}}
  \caption{Visual evaluation results for the pancreatic dataset (FOV: approx. 125 pixels, 56th axial slice). The red box delineates the ROI; the bottom-left corner displays a local magnification, while the bottom-right corner shows the difference map between the algorithm's local magnification and the ground truth. (a) Ground Truth, (b) FDK, (c) Projection Mirror Extension Reconstruction, (d) SIRT, (e) NAF, (f) NeRF, (g) ${R^2}$-Gaussian, (h) SAX-NeRF, (i) Ours.}
  \label{s_fov}
\end{figure*}

\begin{figure*}[!b] 
  \centerline{\includegraphics[width=\linewidth]{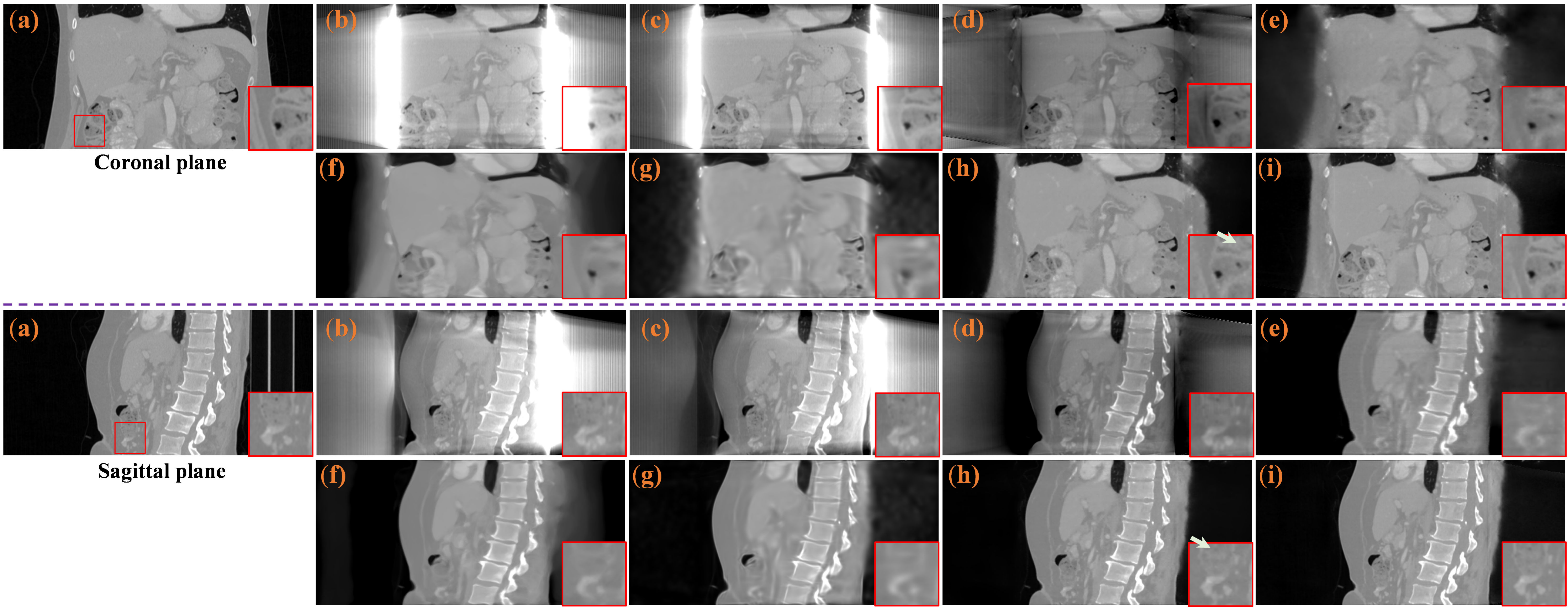}}
  \caption{Reconstruction results in the coronal and sagittal planes using different methods. The bottom right corner displays a magnified view of the red ROI region. (a) Ground Truth, (b) FDK, (c) Projection Mirror Extension Reconstruction, (d) SIRT, (e) NAF, (f) NeRF, (g) ${R^2}$-Gaussian, (h) SAX-NeRF, (i) Ours.}
  \label{3d}
\end{figure*}

\begin{figure*}[htbp] 
  \centerline{\includegraphics[width=\linewidth]{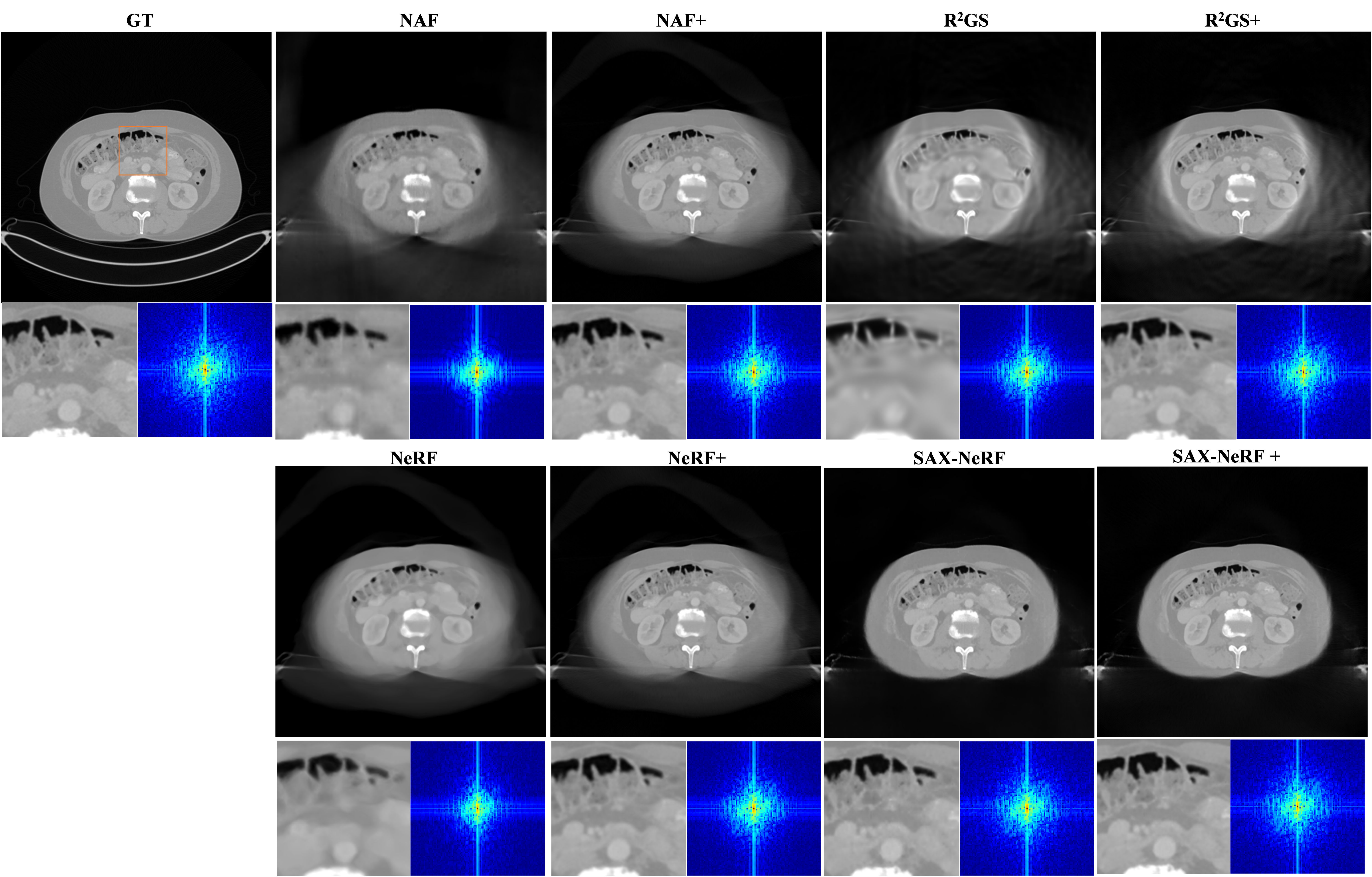}}
  \caption{Compare the improvement effects of different algorithms under the proposed iterative architecture (170th axial slice). The second line shows a magnified view of the area and its corresponding spectrum. "+" indicates that the proposed architecture has been deployed on the baseline approach.}
  \label{fre+}
\end{figure*}

\section{EXPERIMENTS AND RESULTS}
\subsection{Datasets}
We evaluated our method on three publicly available medical datasets (Pancreas, Pelvis, and Abdomen) and one acquired physical dataset. Because the public datasets lack raw projection data, we employed the TIGRE toolbox \cite{biguri2016tigre} to simulate the cone-beam forward projection process. To assess performance under varying degrees of truncation, we configured distinct effective Fields of View (FOVs) for each dataset. Specifically, the FOV radii for the central slices were set to approximately 125, 150, and 203 pixels for the Pancreas, Pelvis, and Abdomen datasets, respectively. For the real-world physical experiment, we acquired truncated projections of an ex vivo ovine bone specimen using a CD-130BX micro-CT scanner (Chongqing Zhence) operating at 60 kV and 40 mA. Comprehensive scanning geometries and parameters for all datasets are detailed in Appendix B.

\subsection{Implementation Details}
The proposed framework was implemented in PyTorch. Our coordinate network architecture utilizes a 4-layer Lineformer, initialized following SAX-NeRF. The network was optimized using the Adam optimizer for 1,500 iterations with a learning rate of $1 \times 10^{-4}$. During volume rendering, we sampled a batch of 1,024 rays per iteration, evaluating 576 points along each ray. The subsequent physics-based iterative refinement module was executed for 500 iterations. All experiments were conducted on a single NVIDIA RTX 4090D GPU.

To comprehensively validate the proposed method, we benchmarked it against four categories of reconstruction techniques: 1) analytical methods, including FDK \cite{feldkamp1984practical} and FDK with projection mirroring extrapolation (EX-FDK) \cite{ohnesorge2000efficient}; 2) iterative methods, specifically SIRT \cite{gilbert1972iterative}; The remaining two categories are unsupervised methods originally designed for sparse-view reconstruction. Although these methods have never been applied to truncated reconstruction before, we observed that they yield impressive results without any structural modification. Therefore, their inclusion serves as a reference for generalizing such algorithms to truncation tasks. These two categories are: 3) coordinate-based implicit neural representations, comprising NAF \cite{zha2022naf}, NeRF \cite{mildenhall2021nerf}, and SAX-NeRF \cite{cai2024structure}; and 4) 3D Gaussian splatting, represented by R$^2$GS \cite{zha2024r}.

To ensure fair comparisons, all baseline methods were implemented strictly according to their original publications and official codebases. For traditional baselines, the mirroring extrapolation utilized a truncation detection threshold of 0.1, while SIRT was executed for 500 iterations using the ASTRA toolbox \cite{van2016fast}. For the neural rendering baselines (NAF, NeRF, and SAX-NeRF), the training length (1,500 iterations) and ray batch size (1,024) were standardized. NAF and SAX-NeRF sampled 576 points per ray using a 4-layer MLP and a 4-layer Lineformer, respectively. Standard NeRF employed an 8-layer MLP with 192 coarse and 192 fine samples per ray; however, due to GPU memory constraints during the high-resolution real-world experiment, the fine samples were reduced to 96. Finally, R$^2$GS was similarly optimized for 1,500 iterations. These configurations were kept consistent across both simulated and physical scenarios.

Reconstruction quality was assessed both qualitatively and quantitatively. Qualitative evaluations highlight structural fidelity and artifact suppression through magnified ROIs and error difference maps. Quantitative evaluations utilize the Peak Signal-to-Noise Ratio (PSNR) and Structural Similarity Index Measure (SSIM) \cite{wang2004image}. To thoroughly analyze the extrapolation capabilities, these quantitative metrics were computed separately for the effective region strictly within the FOV and the globally extrapolated volume.

\begin{table}[t]
\centering
\caption{Quantitative results on baseline improvements from the proposed iterative refinement architecture. The `+' suffix denotes the baseline equipped with our proposed architecture.}
\label{tab:method_improvements}
\begin{tabular}{lcc}
\toprule
Method & PSNR (dB) $\uparrow$ & SSIM $\uparrow$ \\
\midrule
NAF            & 24.27 & 0.54 \\
NAF+           & 24.52 & 0.64 \\
\midrule
NeRF           & 26.10 & 0.68 \\
NeRF+          & 26.28 & 0.72 \\
\midrule
R$^2$GS        & 25.34 & 0.71 \\
R$^2$GS+       & 25.50 & 0.72 \\
\midrule
SAX-NeRF       & 27.37 & 0.78 \\
SAX-NeRF+      & 27.58 & 0.81 \\
\bottomrule
\end{tabular}
\end{table}

\begin{table*}[t]
\centering
\caption{Quantitative evaluation results on Pancreas, Pelvis, and Abdomen datasets. Metrics are reported as mean $\pm$ standard deviation for both the in FOV and the Overall volume.}
\label{tab:quantitative_metrics}
\resizebox{\textwidth}{!}{
\begin{tabular}{llcccccc}
\toprule
\multirow{2}{*}{Method} & \multirow{2}{*}{Metric} & \multicolumn{2}{c}{Pancreas} & \multicolumn{2}{c}{Pelvis} & \multicolumn{2}{c}{Abdomen} \\
\cmidrule(lr){3-4} \cmidrule(lr){5-6} \cmidrule(lr){7-8}
 & & In FOV (125) & Overall & In FOV (150) & Overall & In FOV (203) & Overall \\
\midrule
\multirow{2}{*}{FDK}      & PSNR $\uparrow$ & $16.37 \pm 0.31$ & $9.24 \pm 0.80$ & $20.73 \pm 0.60$ & $13.42 \pm 0.61$ & $16.06 \pm 0.86$ & $11.75 \pm 1.27$ \\
                          & SSIM $\uparrow$ & $0.8275 \pm 0.0003$ & $0.2209 \pm 0.0005$ & $0.8777 \pm 0.0003$ & $0.3154 \pm 0.0011$ & $0.8270 \pm 0.0026$ & $0.5089 \pm 0.0003$ \\
\midrule
\multirow{2}{*}{EX-FDK}   & PSNR $\uparrow$ & $25.32 \pm 0.90$ & $11.16 \pm 0.37$ & $28.61 \pm 0.11$ & $15.57 \pm 0.15$ & $31.85 \pm 0.70$ & $14.83 \pm 0.29$ \\
                          & SSIM $\uparrow$ & $0.9474 \pm 0.0003$ & $0.2686 \pm 0.0004$ & $0.9240 \pm 0.0004$ & $0.3444 \pm 0.0000$ & $0.9365 \pm 0.0001$ & $0.6052 \pm 0.0000$ \\
\midrule
\multirow{2}{*}{SIRT}     & PSNR $\uparrow$ & $21.50 \pm 1.65$ & $18.76 \pm 0.43$ & $26.42 \pm 0.36$ & $22.35 \pm 0.31$ & $31.07 \pm 4.89$ & $22.85 \pm 0.23$ \\
                          & SSIM $\uparrow$ & $0.9209 \pm 0.0007$ & $0.3663 \pm 0.0009$ & $0.9183 \pm 0.0000$ & $0.4404 \pm 0.0002$ & $0.9165 \pm 0.0000$ & $0.6746 \pm 0.0003$ \\
\midrule
\multirow{2}{*}{NAF}      & PSNR $\uparrow$ & $32.33 \pm 1.23$ & $24.43 \pm 0.15$ & $38.75 \pm 0.31$ & $30.45 \pm 0.22$ & $33.97 \pm 2.52$ & $26.29 \pm 1.65$ \\
                          & SSIM $\uparrow$ & $0.9305 \pm 0.0002$ & $0.5461 \pm 0.0002$ & $0.9503 \pm 0.0000$ & $0.8302 \pm 0.0008$ & $0.8688 \pm 0.0004$ & $0.6893 \pm 0.0010$ \\
\midrule
\multirow{2}{*}{NeRF}     & PSNR $\uparrow$ & $36.05 \pm 0.53$ & $26.17 \pm 0.47$ & $36.51 \pm 0.60$ & $30.93 \pm 0.12$ & $32.29 \pm 0.89$ & $26.53 \pm 0.39$ \\
                          & SSIM $\uparrow$ & $0.9510 \pm 0.0002$ & $0.6793 \pm 0.0003$ & $0.9317 \pm 0.0000$ & $0.8725 \pm 0.0001$ & $0.8386 \pm 0.0002$ & $0.7321 \pm 0.0003$ \\
\midrule
\multirow{2}{*}{R2GS}     & PSNR $\uparrow$ & $36.68 \pm 1.84$ & $25.37 \pm 0.10$ & $37.66 \pm 1.04$ & $30.20 \pm 1.69$ & $34.88 \pm 2.70$ & $27.46 \pm 2.28$ \\
                          & SSIM $\uparrow$ & $0.9347 \pm 0.0002$ & $0.7039 \pm 0.0005$ & $0.9421 \pm 0.0000$ & $0.8156 \pm 0.0008$ & $0.8832 \pm 0.0003$ & $0.7529 \pm 0.0006$ \\
\midrule
\multirow{2}{*}{SAX-NeRF} & PSNR $\uparrow$ & $43.41 \pm 1.01$ & $27.59 \pm 0.22$ & $41.87 \pm 0.75$ & $30.83 \pm 0.10$ & $37.95 \pm 2.11$ & $30.27 \pm 1.96$ \\
                          & SSIM $\uparrow$ & $0.9762 \pm 0.0000$ & $0.7836 \pm 0.0001$ & $0.9700 \pm 0.0000$ & $0.8221 \pm 0.0001$ & $0.9135 \pm 0.0002$ & $0.8406 \pm 0.0006$ \\
\midrule
\multirow{2}{*}{Ours}     & PSNR $\uparrow$ & $\mathbf{45.41 \pm 4.03}$ & $\mathbf{27.75 \pm 0.23}$ & $\mathbf{45.67 \pm 3.12}$ & $\mathbf{31.26 \pm 0.22}$ & $\mathbf{41.59 \pm 1.33}$ & $\mathbf{31.01 \pm 2.43}$ \\
                          & SSIM $\uparrow$ & $\mathbf{0.9914 \pm 0.0000}$ & $\mathbf{0.8154 \pm 0.0002}$ & $\mathbf{0.9893 \pm 0.0000}$ & $\mathbf{0.8353 \pm 0.0003}$ & $\mathbf{0.9635 \pm 0.0000}$ & $\mathbf{0.8695 \pm 0.0004}$ \\
\bottomrule
\end{tabular}
}
\end{table*}

\begin{figure*}[htbp] 
  \centerline{\includegraphics[width=0.8\linewidth]{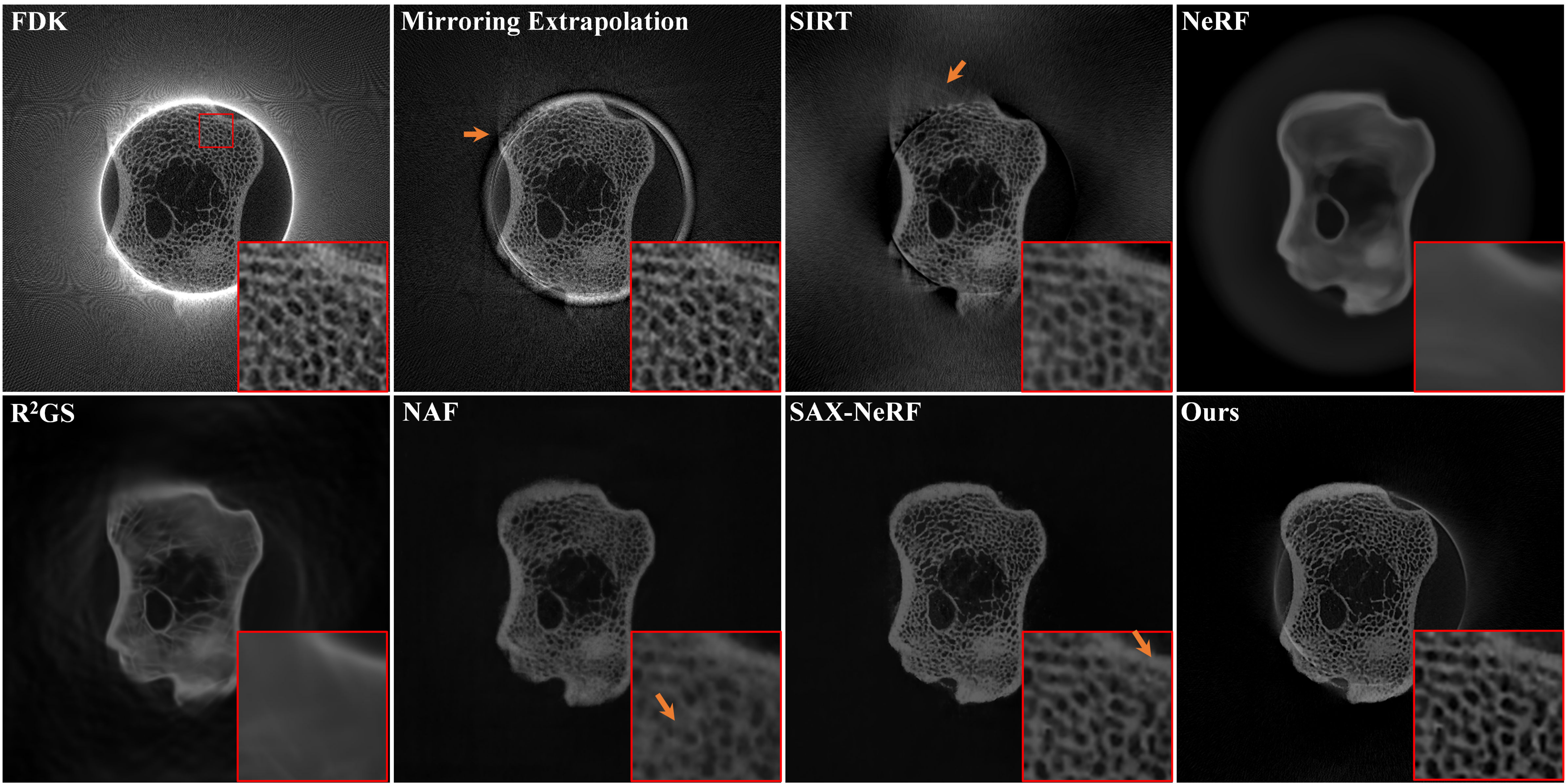}}
  \caption{Results of the visual assessment of the horizontal plane in the real-world experiment (199th axial slice). The magnified view of the ROI is located in the bottom right corner.}
  \label{realxy}
\end{figure*}

\begin{figure*}[htbp] 
  \centerline{\includegraphics[width=0.8\linewidth]{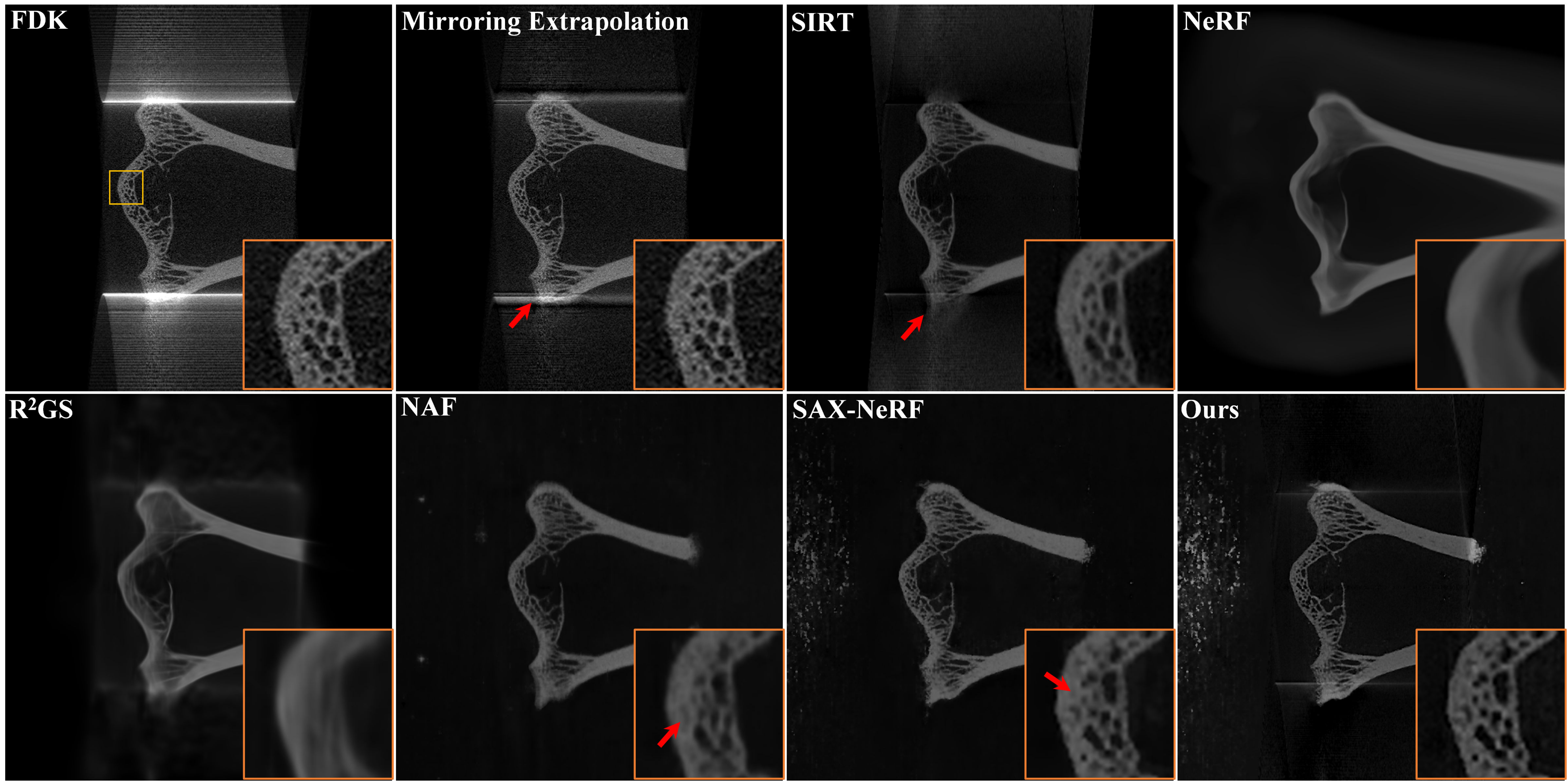}}
  \caption{Results of the sagittal-plane visual assessment from the real-world experiment. An enlarged view of the ROI is located in the bottom right corner.}
  \label{realyz}
\end{figure*}

\subsection{Simulation Experiments at Different FOVs}
To evaluate the capabilities of different models in removing truncation artifacts and extrapolating data across various datasets and FOV sizes. 

Fig. \ref{s_fov} provides a visual comparison of representative reconstructed slices generated by various methods under severe data truncation. The FDK reconstruction suffers from severe truncation artifacts, most notably pronounced bright peripheral rings and a significant shift in attenuation values (CT numbers). Furthermore, the spinous process within the red ROIs is rendered nearly invisible. While projection mirror extension techniques partially mitigate these bright artifacts, the spinous process remains heavily blurred in the magnified ROI, and the attenuation value drift persists. Similarly, the classic iterative SIRT algorithm suppresses truncation-induced rings but leaves substantial artifacts outside the FOV and fails to correct the attenuation drift within the FOV.
In contrast, coordinate-based methods—such as those based on NeRF and 3D Gaussian Splatting—effectively eliminate these truncation artifacts. By bypassing the traditional back-projection operation, these forward-rendering paradigms inherently avoid the Gibbs phenomenon typically triggered by sharp, truncated projection boundaries. For instance, the NAF method produces ring-free reconstructions with an expanded effective FOV compared to FDK; however, close inspection of the magnified ROI reveals degraded bone margins and severe blurring in the soft tissue regions (yellow arrow). Standard NeRF demonstrates superior extrapolation coverage relative to NAF, yet still suffers from significant over-smoothing along the osseous boundaries. Although 3D Gaussian-based methods are highly regarded for their rendering efficiency in sparse reconstruction, they perform poorly under severe truncation, as evidenced by the complete loss of bone contours in the region denoted by the yellow arrow.
Among the state-of-the-art baselines, SAX-NeRF demonstrates robust performance, offering extensive extrapolation and improved detail preservation, likely due to its linear attention mechanism. Nevertheless, the corresponding difference map reveals persistent edge-related artifacts and detail loss (yellow arrow), confirming a deficit in high-frequency recovery. Furthermore, SAX-NeRF fails to resolve the clinically critical, subtle textural details within the relatively flat regions (red arrow).
As discussed in the Methodology section, standard coordinate-based approaches inherently suffer from spectral bias, prioritizing low-frequency structures over high-frequency details. In contrast, our proposed method successfully overcomes this limitation. While maintaining an extrapolation range comparable to SAX-NeRF, our approach achieves a significant enhancement in textural fidelity. Both the magnified views and difference maps demonstrate strong alignment with the GT. This superior performance is directly attributable to our iterative refinement module, which explicitly extracts high-frequency residual information from the original projections and re-injects it into the volume, ensuring the preservation of diagnostically vital micro-structures.

To demonstrate the robustness and generalizability of the proposed method across varying FOVs and diverse anatomical domains, we conducted additional evaluations using pelvic and abdominal datasets. Representative reconstructions are provided in Fig. \ref{m_fov} and \ref{l_fov} of Appendix C. Aligning with our findings from the pancreatic dataset, the proposed approach consistently delivers superior visual fidelity; it achieves extensive spatial extrapolation while meticulously preserving fine-grained textural details.

Table \ref{tab:quantitative_metrics} details the quantitative evaluation on the Pancreas, Pelvis, and Abdomen datasets, assessing both the effective field of view (In FOV) and the extrapolated global volume (Overall). Our proposed structure-constrained approach consistently achieves state-of-the-art performance across all metrics and regions.
Traditional analytical and iterative methods exhibit inherent limitations under severe data truncation, struggling to exceed 20 dB in Overall PSNR on the Pancreas dataset. While coordinate-based implicit rendering methods (e.g., NAF, NeRF, $R^2$GS) significantly improve global accuracy by circumventing the back-projection process, our approach further establishes a new benchmark.
Specifically, our method substantially outperforms the strongest baseline, SAX-NeRF, particularly in recovering high-fidelity details within the FOV. On the Pancreas dataset, our method achieves an In-FOV PSNR of 45.41 dB (a 2.0 dB improvement over SAX-NeRF). This advantage expands on other anatomies, yielding In-FOV PSNRs of 45.67 dB (Pelvis) and 41.59 dB (Abdomen), which surpass SAX-NeRF by impressive margins of 3.84 dB and 3.64 dB, respectively.
Furthermore, our method demonstrates exceptional global extrapolation capabilities. On the Abdomen dataset, it achieves an Overall SSIM of 0.8695. This confirms that the approach not only substantially expands the effective FOV but also mitigates spectral bias, seamlessly unifying local textural fidelity with global anatomical coherence.

\subsection{Comparison of Methods for 3D Reconstruction}
Fig. \ref{3d} presents the reconstruction results obtained using different methods in the sagittal and coronal planes. These results demonstrate that both the Neural Radiance Fields-based method and the Gaussian Splatting-based method successfully address the 3D truncation problem without introducing additional inter-slice discontinuity artifacts. However, compared to the SAX-NeRF method, our approach yields clearer structural details, as indicated by the light green arrows.

\subsection{Real World Experments}
Fig. \ref{realxy} illustrates the reconstructed axial slices derived from the real-world experiment. Unlike the simulated phase, this evaluation directly utilizes raw, severely truncated projections acquired from the physical imaging system. The standard FDK reconstruction suffers from prominent truncation-induced ring artifacts. While the projection mirroring extension partially mitigates these concentric rings, the artifacts remain visually disruptive and completely obscure anatomical structures outside the effective FOV. Although the iterative SIRT algorithm successfully suppresses these ring artifacts, the inherent data loss from truncation results in severe structural distortions in the extrapolated (out-of-FOV) regions, as indicated by the yellow arrows.
Coordinate-based implicit rendering methods, such as NeRF and R2GS, effectively eliminate the ring artifacts and achieve a degree of out-of-FOV recovery; however, they suffer from a severe loss of high-frequency detail, as demonstrated in the magnified red ROIs. The NAF method yields more competitive results, yet the magnified views reveal that the structures indicated by the yellow arrows still suffer from noticeable geometric distortion. Similarly, while SAX-NeRF demonstrates improved structural fidelity, close inspection exposes persistent spatial blurring, particularly in the fine features highlighted by the yellow arrows.
In contrast, the proposed method delivers both an extensive extrapolation range and superior high-resolution detail preservation. We do acknowledge that, in this real-world scenario, our approach introduces faint ring-like artifacts. This phenomenon is likely attributable to numerical inconsistencies between the NeRF-based continuous representation and the discrete SIRT iterative refinement process under physical measurement conditions. Nevertheless, the impact of these artifacts on the overall structural integrity is minor. Overall, the proposed framework demonstrates the most robust and accurate reconstruction performance among all evaluated techniques in the physical experimental setting.

Fig. \ref{realyz} presents the reconstruction results in the sagittal plane from the actual experiment. Similar to the analysis results in the horizontal plane, the proposed method demonstrates superior extrapolation performance and the finest detail resolution.

\subsection{Universality of the Iterative Architecture}
We deployed the proposed architecture on various neural scene representation methods. These methods include NeRF and NAF using classic MLP networks, SAX-NeRF using Transformer networks, and Gaussian splatting. The number of iterations was uniformly set to 500. As shown in the enlarged image of Fig. \ref{fre+}, all methods initially exhibited significant blurring and unclear structures; after iterative refinement, the structures became significantly clearer and the difference from the GT was minimal. Further analysis of their spectral characteristics revealed that before iterative refinement, high-frequency information was severely lost compared to the GT, but after iterative refinement, the high-frequency features were well replenished. Table \ref{tab:method_improvements} shows the global quantitative indicators of the proposed architecture's improvement over different placement methods. The results show that the proposed architecture has significant improvements, especially in the SSIM metric, which confirms the method's ability to improve structural details.

\section{Discussion}
\subsection{Impact of iteration number}
We analyzed the impact of different iteration counts on the results. As shown in the Fig. \ref{iters}, the results improved to some extent with increasing iteration count, and we chose NeRF as the baseline.. Furthermore, the improvement in SSIM was significant, indicating that iterative reconstruction had a substantial effect on structure recovery. However, the effect gradually stabilized with increasing iteration count. Therefore, to balance efficiency and effectiveness, an iteration count of 200 was optimal.

\begin{figure}[htbp] 
  \centerline{\includegraphics[width=0.8\linewidth]{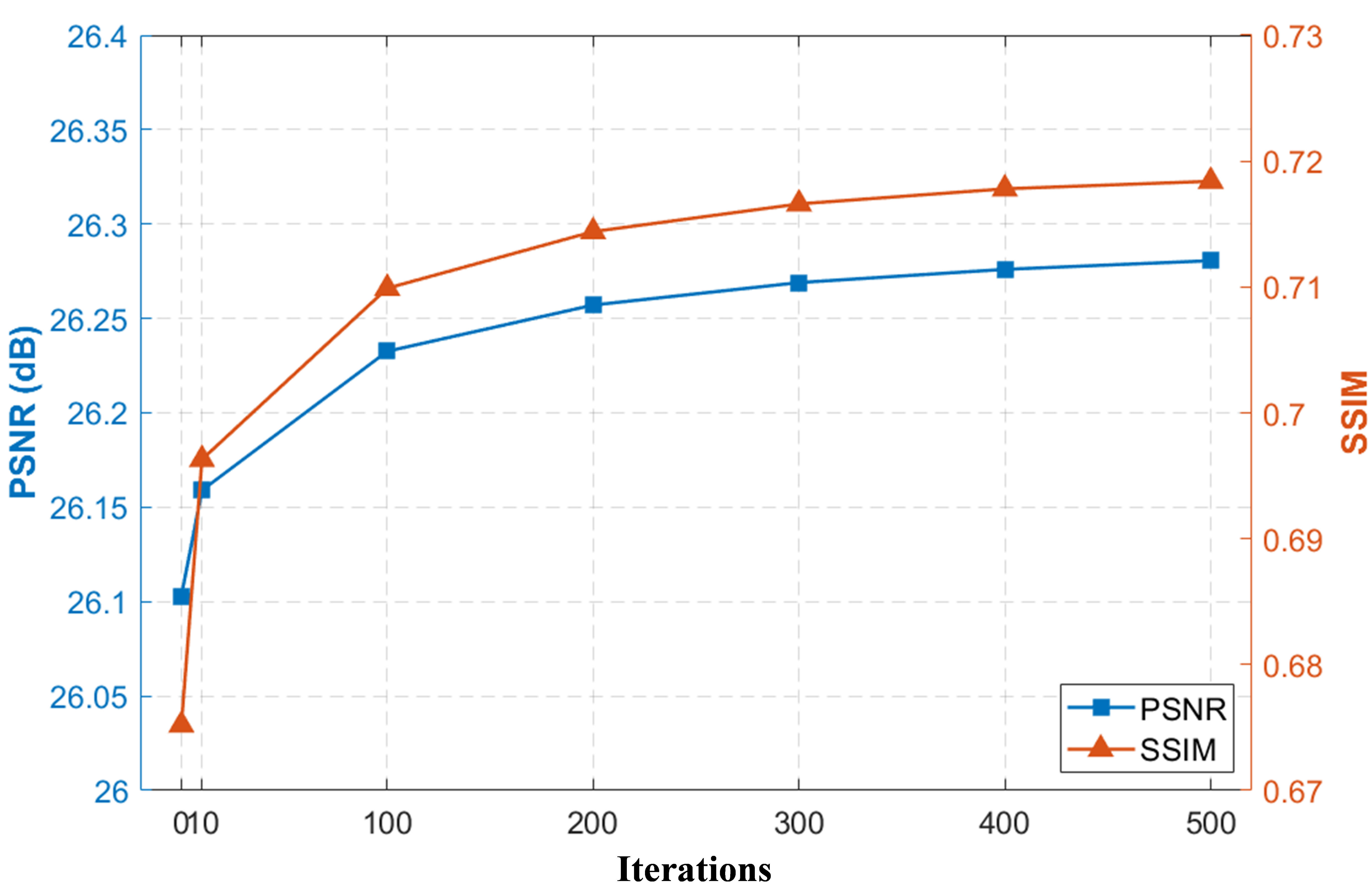}}
  \caption{PSNR and SSIM change curves with the number of iterations of reconstruction.}
  \label{iters}
\end{figure}

\subsection{The Impact of Iterative Reconstruction on Regions Outside the FOV}
While the preceding results demonstrate the efficacy of the proposed iterative refinement architecture in restoring high-frequency details within the FOV, its impact on the extrapolated regions outside the FOV warrants further investigation. To evaluate this, we analyzed the truncated abdominal dataset, comparing the baseline SAX-NeRF against our augmented iterative framework (Fig. \ref{o_fov}). The results reveal a substantial enrichment of high-frequency features in the out-of-FOV regions following iterative refinement. This confirms that the proposed method successfully propagates and supplements critical high-frequency structural information globally, ensuring high-fidelity reconstruction both inside and outside the effective FOV.
\begin{figure}[htbp] 
  \centerline{\includegraphics[width=\linewidth]{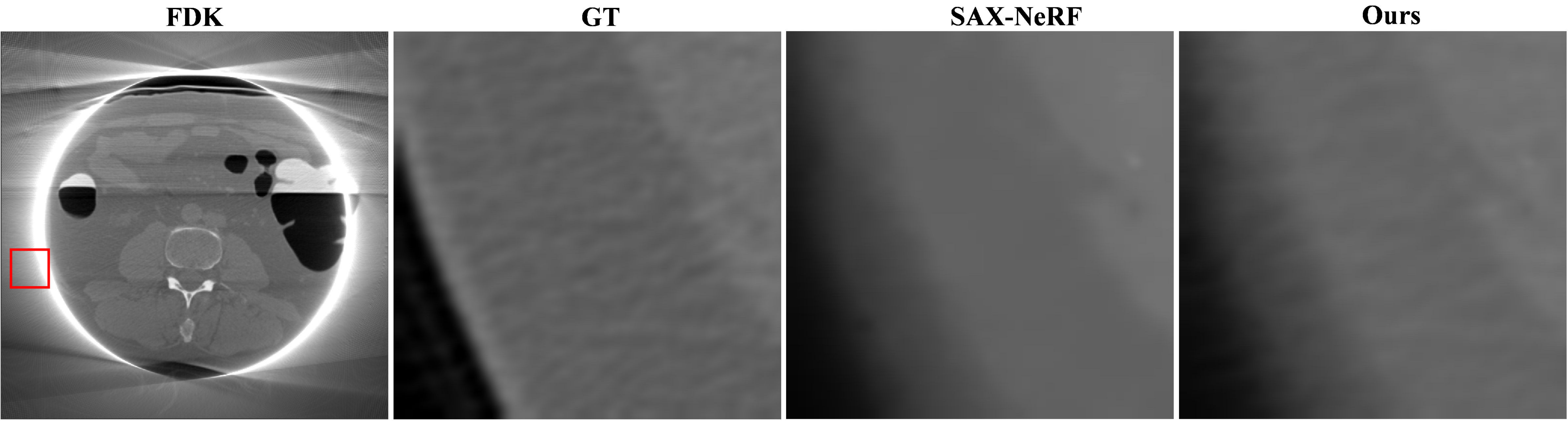}}
  \caption{High-frequency information outside the field of view (FOV) is compared between the baseline method (SAX-NeRF) and the proposed method.}
  \label{o_fov}
\end{figure}

\subsection{Limitations}
While the proposed framework effectively addresses CBCT data truncation in a self-supervised manner, certain limitations warrant discussion. First, computational efficiency remains a significant bottleneck. Our experiments indicate that the most robust coordinate-based baseline relies on the NeRF architecture, which is inherently computationally expensive—requiring approximately ten hours to reconstruct a $512^3$ volume. This prolonged execution limits its immediate applicability in time-sensitive clinical scenarios. Although 3D Gaussian splatting (3DGS) paradigms generally surpass NeRF-based methods in both rendering efficiency and standard reconstruction quality, our empirical findings reveal that 3DGS severely underperforms under truncation conditions, particularly regarding out-of-FOV extrapolation. Consequently, our future work will investigate the fundamental optimization constraints that hinder 3DGS in truncated scenarios, with the goal of integrating its high efficiency into our proposed architecture. Secondly, while the physics-based iterative refinement module successfully recovers high-frequency details, we observed that it introduces weak secondary ring artifacts in real-world experiments. Although these artifacts have a negligible impact on overall structural integrity and diagnostic value, they indicate a slight numerical difference between continuous neural representations and discrete iterative processes in the real world. Future research will focus on developing more robust and seamlessly coupled optimization strategies to completely prevent the generation of these secondary artifacts.

\section{Conclusion}
In this work, we proposed a novel iteratively refined neural scene representation framework for self-supervised truncated CBCT 3D reconstruction The presented method utilizes (i) a forward-rendering scheme to bypass traditional filtering and backprojection operations, thereby fundamentally eliminating truncation-induced ring artifacts, (i) a physics-based iterative refinement module to progressively re-extract and inject missing high-frequency structural details from the raw projections, and (iii) a self-supervised optimization paradigm that entirely eliminates the reliance on un-truncated ground-truth data. Experimental results on both simulated phantoms and real-world physical acquisitions demonstrate the effectiveness of method under data truncation. Notably, our model successfully mitigates the high-frequency losses introduced by coordinate networks, seamlessly combining extensive out-of-field extrapolation with high-fidelity local texture preservation. Future work will focus on more efficient neural field representation methods and designing more robust coupling strategies to improve the clinical application value of the algorithm.

\bibliographystyle{IEEEtran}
\bibliography{ref}

\section{Appendix}

\begin{figure*}[htbp] 
  \centerline{\includegraphics[width=\linewidth]{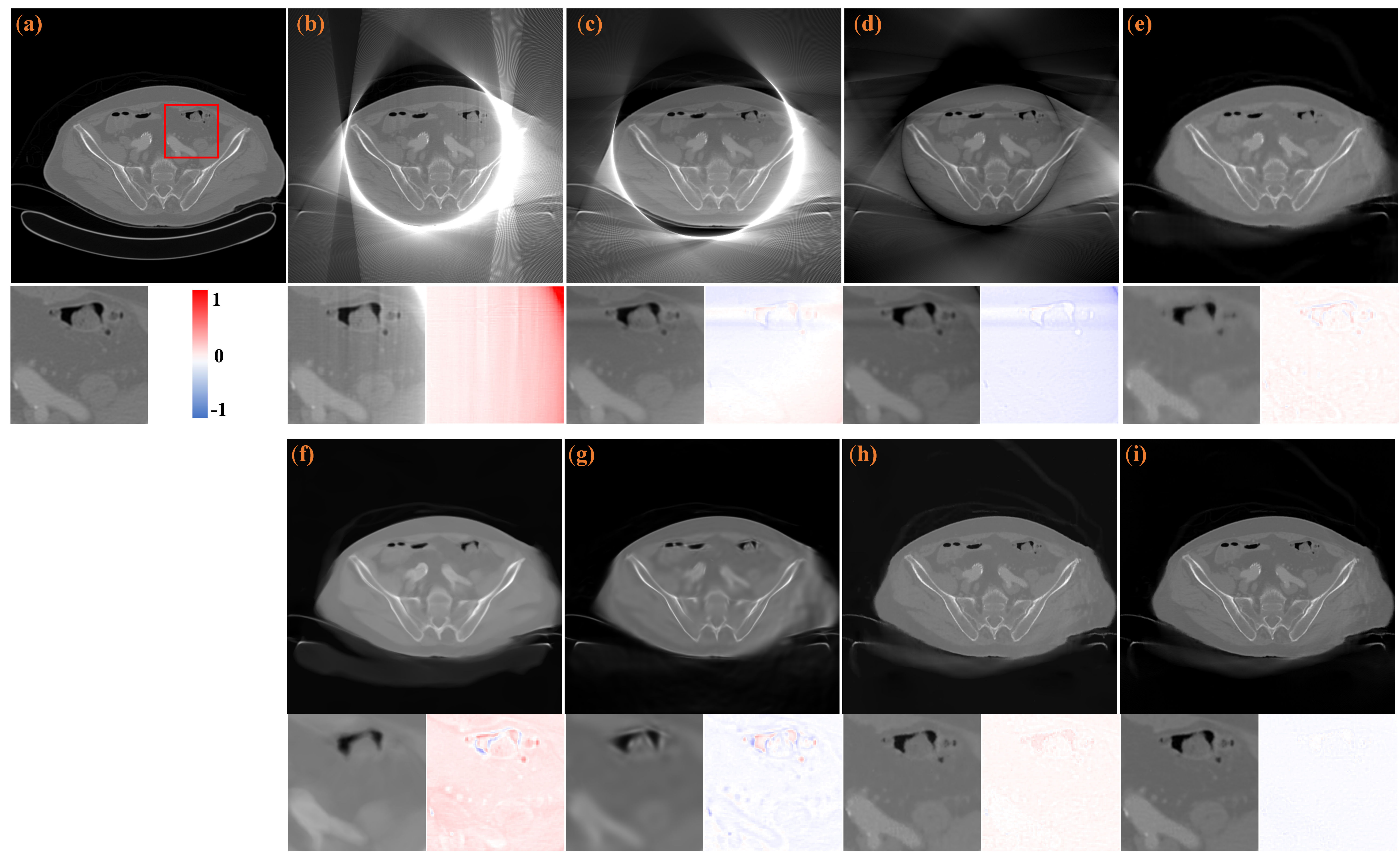}}
  \caption{Visual evaluation results for the pelvis dataset (FOV: approx. 150 pixels, 80th axial slice). The red box delineates the ROI; the bottom-left corner displays a local magnification, while the bottom-right corner shows the difference map between the algorithm's local magnification and the ground truth. (a) Ground Truth, (b) FDK, (c) Projection Mirror Extension Reconstruction, (d) SIRT, (e) NAF, (f) NeRF, (g) ${R^2}$-Gaussian, (h) SAX-NeRF, (i) Ours.}
  \label{m_fov}
\end{figure*}

\begin{table*}[htbp]
\centering
\caption{Scanning parameters for simulation and real data}
\label{tab:scanning_parameters}
\begin{tabular}{lcccc}
\toprule
\multirow{2}{*}{Parameter} & \multicolumn{3}{c}{Simulation Data} & Real Data \\
\cmidrule(lr){2-4} \cmidrule(lr){5-5}
 & Pancreas & Pelvis & Abdomen & Sheep Bone \\
\midrule
Source-to-object distance (SOD) / mm & 1000 & 1000 & 1000 & 150 \\
Source-to-detector distance (SDD) / mm & 1500 & 1500 & 1500 & 600 \\
Detector pixel size / $\text{mm}^2$ & $1.0 \times 1.0$ & $1.0 \times 1.0$ & $1.0 \times 1.0$ & $0.2 \times 0.2$ \\
Detector matrix / pixel & $384 \times 384$ & $384 \times 384$ & $384 \times 384$ & $512 \times 512$ \\
Image matrix / voxel & $512 \times 512 \times 240$ & $512 \times 512 \times 174$ & $512 \times 512 \times 242$ & $512 \times 512 \times 512$ \\
Voxel size / $\text{mm}^3$ & $1.0 \times 1.0 \times 1.0$ & $0.8398 \times 0.8398 \times 0.8398$ & $0.625 \times 0.625 \times 0.625$ & -- \\
Angular range / $^\circ$ & 180 & 180 & 180 & 360 \\
Number of projections & 360 & 360 & 360 & 500 \\
\bottomrule
\end{tabular}
\end{table*}

\begin{figure*}[htbp] 
  \centerline{\includegraphics[width=\linewidth]{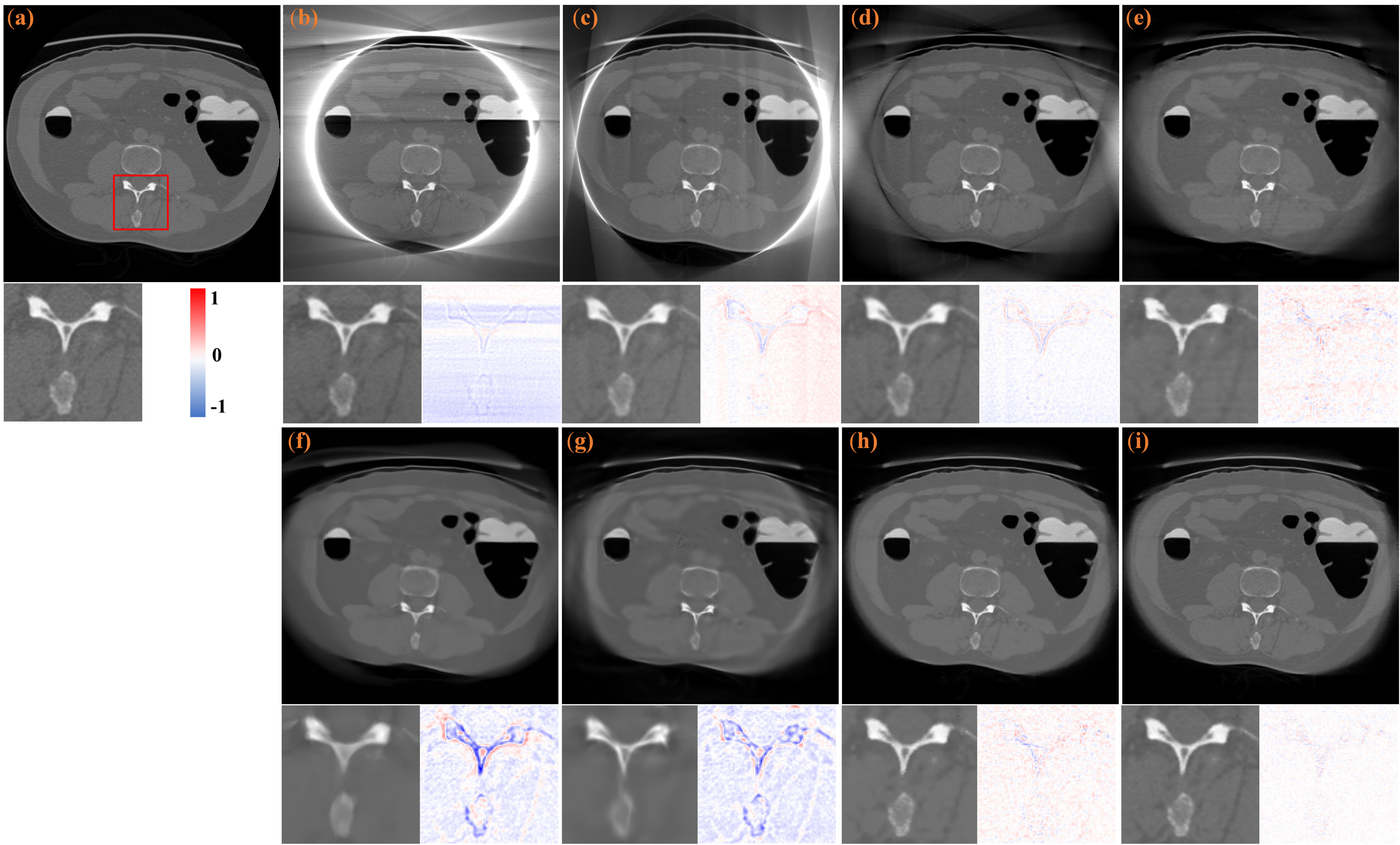}}
  \caption{Visual evaluation results for the abdomen dataset (FOV: approx. 203 pixels, 230th axial slice). The red box delineates the ROI; the bottom-left corner displays a local magnification, while the bottom-right corner shows the difference map between the algorithm's local magnification and the ground truth. (a) Ground Truth, (b) FDK, (c) Projection Mirror Extension Reconstruction, (d) SIRT, (e) NAF, (f) NeRF, (g) ${R^2}$-Gaussian, (h) SAX-NeRF, (i) Ours.}
  \label{l_fov}
\end{figure*}

\section{Proof of Theorem \ref{thm:spectral_concentration}}
\label{appendix:proof_spectral_concentration}

\begin{proof}
Using the SVD decomposition, the transpose of the truncated forward projection matrix is $\mathcal{A}_\mathcal{T}^\top = \mathbf{V} \mathbf{\Sigma} \mathbf{U}^\top$. The initial gradient $\mathbf{g}_0$ of the data fidelity term is computed as:
\begin{equation}
    \mathbf{g}_0 = \mathcal{A}_\mathcal{T}^\top \mathbf{r}_0 = \mathbf{V} \mathbf{\Sigma} \mathbf{U}^\top \mathbf{r}_0.
\end{equation}
Projecting $\mathbf{g}_0$ onto the low-order image subspace $\mathbf{V}_{\text{low}}$ yields:
\begin{equation}
    \mathbf{V}_{\text{low}}^\top \mathbf{g}_0 = \mathbf{V}_{\text{low}}^\top \mathbf{V} \mathbf{\Sigma} \mathbf{U}^\top \mathbf{r}_0.
\end{equation}
Exploiting the orthonormality of $\mathbf{V}$, we have $\mathbf{V}_{\text{low}}^\top \mathbf{V} = [\mathbf{I}_k \quad \mathbf{0}]$, where $\mathbf{I}_k$ is the $k \times k$ identity matrix. Consequently:
\begin{equation}
    \mathbf{V}_{\text{low}}^\top \mathbf{g}_0 = [\mathbf{I}_k \quad \mathbf{0}] \mathbf{\Sigma} \mathbf{U}^\top \mathbf{r}_0 = \mathbf{\Sigma}_k \mathbf{U}_{\text{low}}^\top \mathbf{r}_0,
\end{equation}
with $\mathbf{\Sigma}_k = \operatorname{diag}(\sigma_1, \dots, \sigma_k)$. Taking the $\ell_2$-norm and applying the submultiplicativity of the spectral norm, we obtain:
\begin{equation}
    \|\mathbf{V}_{\text{low}}^\top \mathbf{g}_0\| \le \|\mathbf{\Sigma}_k\|_2 \cdot \|\mathbf{U}_{\text{low}}^\top \mathbf{r}_0\|.
\end{equation}
Because the largest singular value of $\mathbf{\Sigma}_k$ is $\sigma_1$, we have $\|\mathbf{\Sigma}_k\|_2 = \sigma_1$. Invoking Assumption 1 yields the upper bound for the low-frequency gradient energy:
\begin{equation}
    \|\mathbf{V}_{\text{low}}^\top \mathbf{g}_0\| \le \sigma_1 \varepsilon \|\mathbf{r}_0\|.
\label{eq:low_bound}
\end{equation}

Analogously, the projection of the gradient onto the high-order subspace is:
\begin{equation}
    \mathbf{V}_{\text{high}}^\top \mathbf{g}_0 = \mathbf{\Sigma}_{\text{high}} \mathbf{U}_{\text{high}}^\top \mathbf{r}_0.
\end{equation}
Since $\mathbf{r}_0$ can be orthogonally decomposed as $\mathbf{r}_0 = \mathbf{U}_{\text{low}} \mathbf{U}_{\text{low}}^\top \mathbf{r}_0 + \mathbf{U}_{\text{high}} \mathbf{U}_{\text{high}}^\top \mathbf{r}_0$, Pythagoras' theorem implies:
\begin{equation}
    \|\mathbf{U}_{\text{high}}^\top \mathbf{r}_0\|^2 = \|\mathbf{r}_0\|^2 - \|\mathbf{U}_{\text{low}}^\top \mathbf{r}_0\|^2.
\end{equation}
Assumption 1 ($\|\mathbf{U}_{\text{low}}^\top \mathbf{r}_0\| \le \varepsilon \|\mathbf{r}_0\|$) thereby guarantees that:
\begin{equation}
    \|\mathbf{U}_{\text{high}}^\top \mathbf{r}_0\| \ge \sqrt{1 - \varepsilon^2} \|\mathbf{r}_0\|.
\end{equation}
Let $\sigma_r$ be the smallest non-zero singular value in $\mathbf{\Sigma}_{\text{high}}$. The lower bound for the high-frequency gradient energy is given by:
\begin{equation}
    \|\mathbf{V}_{\text{high}}^\top \mathbf{g}_0\| \ge \sigma_r \|\mathbf{U}_{\text{high}}^\top \mathbf{r}_0\| \ge \sigma_r \sqrt{1 - \varepsilon^2} \|\mathbf{r}_0\|.
\label{eq:high_bound}
\end{equation}

To ensure that the high-frequency gradient strictly dominates the low-frequency gradient, we require $\|\mathbf{V}_{\text{high}}^\top \mathbf{g}_0\| > \|\mathbf{V}_{\text{low}}^\top \mathbf{g}_0\|$. Comparing the bounds established in Eq. \ref{eq:low_bound} and Eq. \ref{eq:high_bound}, this holds true if:
\begin{equation}
    \sigma_r \sqrt{1 - \varepsilon^2} > \sigma_1 \varepsilon \implies \frac{\varepsilon}{\sqrt{1 - \varepsilon^2}} < \frac{\sigma_r}{\sigma_1} = \kappa^{-1},
\end{equation}
where $\kappa$ is the effective condition number of the truncated forward operator. Because the spectral bias of the NeSR naturally drives $\varepsilon \to 0$ for low-frequency structures, this condition is readily satisfied in practice, completing the proof.
\end{proof}

\subsection{Scan Parameters} 
The scanning parameters for the simulated and actual data are shown in Table \ref{tab:scanning_parameters}.

\subsection{Reconstruction Results from Different Methods} 
Fig. \ref{m_fov} and \ref{l_fov} show the results of truncation and reconstruction of the pelvis and abdomen using different methods.

\vspace{12pt}

\end{document}